\documentclass{new_tlp}
\usepackage{amsmath}
\usepackage{amssymb}
\usepackage{multirow}
\usepackage{amsfonts}
\usepackage{graphicx}
\usepackage{color}
\usepackage{url}

\newcommand{\as}{``}

\newcommand{\be}{\begin{em}}
	\newcommand{\ee}{\end{em}}
\newcommand{\bb}{\begin{bf}}
	\newcommand{\eb}{\end{bf}}
\newcommand{\I}[1]{\relax\ifmmode\mbox{\it#1}\else{\it#1}\fi}

\newcommand{\tbm}{\hspace*{8mm}}

\newcommand{\no}{not\,}

\newcommand{\rif}{~\ref}

\newcommand{\K}{\mbox{\textbf{K}}}
\newcommand{\M}{\mbox{\textbf{M}}}
\newcommand{\N}{\mbox{\textbf{\no}}\,}

\def\naf{not\,}
\def\ar{\leftarrow}

\newcommand{\Bodyobj}{\mathit{Body_{obj}}}
\newcommand{\Bodysubj}{\mathit{Body_{subj}}}

\newtheorem{definition}{Definition}[section]
\newtheorem{proposition}{Proposition}[section]
\newtheorem{corollary}{Corollary}[section]

\newtheorem{theorem}{Theorem}[section]

\newtheorem{property}{Property}[section]

\title[ELP: study of some properties]{Epistemic Logic Programs:\\ a study of some properties~\thanks{Research partially performed
within the activities of the Action COST CA17124 ``DigForASP'',
the Interdepartmental Project on AI (Strategic Plan UniUD–22-25),
and the INdAM-GNCS project CUP E53C22001930001.}}

\author[S. Costantini and A. Formisano]{Stefania Costantini
\\DISIM - Universit{\`a} dell'Aquila, via Vetoio, L'Aquila, Italy\\
	Gruppo Nazionale per il Calcolo Scientifico - INdAM, Roma, Italy\\
	\email{stefania.costantini@univaq.it}
\and
	Andrea Formisano
\\DMIF - Universit{\`a} di Udine, via delle Scienze 206, Udine, Italy\\
Gruppo Nazionale per il Calcolo Scientifico - INdAM, Roma, Italy\\
	\email{andrea.formisano@uniud.it}
}

\begin{document}
\maketitle

\begin{abstract}
	Epistemic Logic Programs (ELPs), extend Answer Set Programming (ASP) with epistemic operators. The semantics of such programs is provided in terms of \emph{world views}, which are sets of belief sets, i.e., syntactically, sets of sets of atoms. Different semantic approaches propose different characterizations of world views. Recent work has introduced semantic properties that should be met by any semantics for ELPs, like the \emph{Epistemic Splitting Property}, that, if satisfied, allows to modularly compute world views in a bottom-up fashion, analogously to ``traditional'' ASP. We analyze the possibility of changing the perspective, shifting from a bottom-up to a top-down approach to splitting. We propose a basic top-down approach, which we prove to be equivalent to the bottom-up one. We then propose an extended approach, where our new definition: (i) is provably applicable to many of the existing semantics; (ii) operates similarly to \as traditional'' ASP; (iii) provably coincides under any semantics with the bottom-up notion of splitting at least on the class of \emph{Epistemically Stratified Programs} (which are, intuitively, those where the use of epistemic operators is stratified); (iv) better adheres to common ASP programming methodology.

	Under consideration in Theory and Practice of Logic Programming (TPLP)
\end{abstract}


\section{Introduction}

Epistemic Logic programs (ELPs, in the following just \emph{programs}, if not explicitly stated differently), were first introduced in \cite{Gelfond91,Gelfond94}, and extend Answer Set Programs, defined under the Answer Set Semantics \cite{GelLif88}, with \emph{epistemic operators} that are able to introspectively \as look inside'' a program's own semantics, which is defined in terms of its \emph{answer sets} (cf.\ \cite{FandinnoFG22} for a historical review of research on this topic). In fact, $\K A$ means that (ground) atom $A$ is true in every answer set of the program $\Pi$ where $\K A$ occurs. Related operators that can be defined in terms of $\K$ are the \emph{possibility operator} $\M$ (not treated in this paper) where $\M A$ means that $A$ is true in some of the answer sets of $\Pi$, and the \emph{epistemic negation operator} \N\!\!, where $\N A$ expresses that $A$ \emph{is not provably true}, meaning that $A$ is false in at least one answer set of $\Pi$.

The semantics of ELPs is provided in terms of \emph{world views}: instead of a unique set of answer sets (a unique \as world view'' in the new terminology) like in Answer Set Programming (ASP), there is now a set of such sets. Each world view consistently satisfies (according to a given semantics) the epistemic expressions that appear in a given program. 
Many semantic approaches 
	for ELPs have been introduced beyond the seminal work of Gelfond and Przymusinska \cite{Gelfond91}, among
which we mention \cite{Gelfond11,Truszczynski11,CerroHS15,ShenE16,Kahl18,Su19,CabalarFC19,Costantini022,Su2022}. 

Recent work 
	extends to Epistemic Logic Programming notions that have already been defined for ASP, and that
	 might prove useful in ELPs as well.
	In particular, Cabalar et al.\ consider \emph{splitting} (introduced for ASP in \cite{Lif94}), which allows a program to be seen as divided (\as split'') into two parts, the \as top'' and \as bottom'' in a principled way, i.e., atoms occurring in the bottom can occur only in the body of rules in the top. This allows the answer sets of the program to be computed incrementally, in the following way: compute the answer sets of the bottom part, and use them (one by one) to simplify the top part; then, compute the answer sets of the simplified top part; finally, the answer sets of the overall program are obtained as the union of each answer set of the bottom with the corresponding answer sets of the simplified top (such a procedure can be iterated, i.e., the top and the bottom could in turn be split). Cabalar et al.\ then extend to ELPs the concept of splitting and the method of incremental calculation of the semantics (here, it is the world views that must be calculated).
This is achieved by defining a notion of \emph{Epistemic Splitting}, where top and bottom are defined with respect to the occurrence of
	epistemic operators, and a corresponding \emph{Epistemic Splitting Property}, which is fulfilled by a semantics if it allows the world views to be computed bottom-up (a precise definition is seen below). Further, Cabalar et al.\ adapt properties of ASP to ELPs, which are implied by this property, namely, the fact that adding constraints leads to reduce the number of answer sets (\emph{Subjective Constraint Monotonicity}), and \emph{Foundedness}, meaning that atoms composing answer sets cannot have been derived through cyclic positive dependencies. Finally, they define the class of \emph{Epistemically Stratified Programs} that, according to \cite[Th.~2]{CabalarFC21}, admit a unique world view (these programs are those where, intuitively, the use of epistemic operators is stratified). 
In substance, Cabalar et al.\ establish the properties that in their view a semantics should fulfill, and then they compare the existing semantics with respect to these properties. 
	
In this paper, we explore a different stance: we analyze the possibility of changing the perspective about how to exploit a splitting, shifting from a bottom-up to a top-down approach. This applies in the first place to the Epistemic Splitting Property, of which we propose a reformulation allowing world views to be computed top-down. We then propose a substantial extension of the Epistemic Splitting Property, leading to a new approach that:
\begin{itemize}
	\item[(i)] is applicable to many of the existing semantics, while few of them fulfill the Epistemic Splitting Property as originally formulated;
	\item[(ii)] operates similarly to splitting in \as traditional'' ASP;
	\item[(iii)] provably coincides under any semantics with the bottom-up notion of splitting on a significant class of programs, including at least those which are \emph{epistemically stratified};
	\item[(iv)] is compatible with common ASP programming practice, where one defines a problem solution (that would constitute the top)
	that will be merged with a problem instance (that would constitute the bottom).
\end{itemize}

The paper is organized as follows. In Sections\rif{asp} and\rif{elps} we recall ASP and ELPs.
Section\rif{elp-p} reports some definitions from \cite{CabalarFC21} concerning useful properties of ELPs.
         In Section\rif{observations} we introduce some observations on ELPs that lead to formulate our proposal,
         treated in detail in Section\rif{sect:proposal}. In Section\rif{theorem} we state our main theorem and a relevant corollary. Finally, in Section\rif{conclusions} we conclude.

\section{Answer Set Programming and Answer Set Semantics}
\label{asp}

One can see an answer set program (for short, ASP program) as a set of
statements that specify a problem, where each answer set
represents a solution compatible with this specification. A \emph{consistent} ASP program has one or more answer sets, while an \emph{inconsistent} one has no answer sets, meaning that no solution can be found. 
Several well-developed freely available \emph{answer set solvers} exist 
that compute the answer sets of a given program. 
Syntactically, an ASP program $\Pi$ is a
collection of \emph{rules} of the form
$$
A_1 | \ldots | A_g \leftarrow\; L_{1} , \ldots , L_{n}.
$$
where each $A_i$, $0 \leq i \leq g$, is an atom and $|$ indicates disjunction, 
and the $L_i$s, $0\leq i \leq n$, are literals (i.e., atoms or negated atoms of the form $\naf\,A$).
The left-hand side and the right-hand side of the rule are called
\emph{head} and \emph{body}, respectively.

A rule with an empty body is called a \emph{fact}.
As usual, the symbols $\top$ and $\bot$ denote the true and the false Boolean constants, respectively.
The notation $A\,|\,B$ indicates disjunction, usable only in rule heads and, so, in facts.

A rule with an empty head (or, equivalently, with head $\bot$), of the form ~\(\leftarrow L_1,...,L_n.\)~ or ~\(\bot \leftarrow L_1,...,L_n.\)\,, is
a \emph{constraint}, stating that literals $L_1,\ldots,L_n$ are not allowed to be simultaneously true in any answer set; the impossibility of fulfilling such kind of requirement is one of the reasons that makes a program inconsistent.

All extensions of ASP not explicitly mentioned above are not considered in this paper. 
We implicitly refer to the
\emph{ground} version of $\Pi$, which is obtained by replacing in all possible ways the variables occurring in $\Pi$ with the constants occurring in $\Pi$ itself,
and is thus composed of ground atoms, i.e., atoms that contain no variables. 

The \emph{answer set} (or \emph{stable model}) semantics can be defined in several ways \cite{Lifschitz10,CostantiniF15}. However, answer sets of a program $\Pi$, if any exists, are the supported minimal classical models of the program interpreted as a first-order theory in an obvious way. The original definition from \cite{GelLif88}, introduced for programs where rule heads were limited to be single atoms, was in terms of the \emph{GL-Operator}~$\Gamma$.
Given set of atoms $I$ and program $\Pi$,
$\Gamma_{\Pi}(I)$ is defined as the least Herbrand model of the program $\Pi^I$, namely, the Gelfond-Lifschitz reduct of $\Pi$ w.r.t.~$I$.
The program $\Pi^I$ is obtained from $\Pi$ by:
\begin{itemize}
	\item[1.]
		removing all rules which contain a negative literal\, $\no{}A$\, such that $A \in I$; ~ and
	\item[2.]
		removing all negative literals from the remaining rules.
\end{itemize}
Since $\Pi^I$ is a positive program,  the least Herbrand model is guaranteed to exist and can be computed via the standard immediate consequence operator \cite{lloyd87}.
Then, $I$ is an answer set whenever $\Gamma_{\Pi}(I) = I$.

This definition is then extended to the general case,
involving disjunctive heads, by defining $I$ to be an answer set of $\Pi$
if it is a minimal model (w.r.t.\ set inclusion) of $\Pi^I$.

\section{Epistemic Logic Programs}
\label{elps}

Epistemic Logic Programs (ELPs) extend the syntax of ASP programs by introducing, in the body of rules, so-called \emph{subjective literals} (w.r.t. the usual\ \emph{objective literals}).\footnote{Nesting of subjective literals is not considered here.}
Such new literals are constructed via the \emph{epistemic operator} \K\ (disregarding without loss of generality the other epistemic operators). An ELP program is called \emph{objective} if no subjective literals occur therein, i.e., it is an ASP program.
A constraint involving (also) subjective literals is called a \emph{subjective constraint}, whereas one involving objective literals only is an \emph{objective constraint}. 

Let $At$ be the set of atoms occurring (within either objective or subjective literals) in a given program $\Pi$, and $\mathit{Atoms}(r)$ be the set of atoms occurring in rule~$r$. 
By some abuse of notation, we denote by $\mathit{Atoms}(X)$ the set of atoms occurring in $X$, whatever $X$ is (a rule, a program, an expression, etc.).
Let $\mathit{Head}(r)$ be the head of rule $r$ and $\Bodyobj(r)$ (resp., $\Bodysubj(r)$) be the (possibly empty) set of objective (resp., subjective) literals occurring in the body of $r$.
For simplicity, we often write $\mathit{Head}(r)$ and $\Bodyobj(r)$ in place of $\mathit{Atoms}(\mathit{Head}(r))$ and $\mathit{Atoms}(\Bodyobj(r))$, respectively, when the intended meaning is clear from the context. 
We call \emph{subjective rules} those rules whose body is made of subjective literals only.

Literal $\K A$ intuitively means that the (ground) atom $A$ is true in every answer set of the given program $\Pi$ (it is a \emph{cautious consequence} of $\Pi$). Since, as it turns out, whatever the semantic account one will choose there can be several sets of answer sets (called \emph{world views}), the actual meaning of $\K A$ is that $A$ is true in every answer set of some world view of $\Pi$. Each world view thus determines the truth value of
all subjective literals in a program. There are several semantic approaches to ELPs, dictating in different ways how one finds the world views of a given program. Although all such approaches provide the same results in a set of basic examples, they (obviously) differ in others.

Formally, a semantics $\cal{S}$ is a function mapping an ELP program into sets of world views,
i.e., sets of sets of objective literals, where if $\Pi$ is an objective program, then the unique member of ${\cal{S}}(\Pi)$
is the set of stable models of $\Pi$.
Otherwise, each member of ${\cal{S}}(\Pi)$ is an \emph{${\cal{S}}$-world view} of~$\Pi$.
(We will often write ``world view'' in place of ``${\cal{S}}$-world view'' whenever mentioning the specific semantics will be irrelevant.)
For an ${\cal{S}}$-world view $W$ and a literal $\K L$, we write $W \models \K L$ if $L$ is true in all elements of~$W$.\label{defdimodels}

For instance, for program $\{a{\,\ar\,}\no b,~ b{\,\ar\,}\no a,~ e{\,\ar\,}\no \K f,~ f{\,\ar\,}\no\K e\}$, every semantics returns two world views:
 $\{\{a, e\}, \{b, e\}\}$, where $\K e$ is true and $\K f$ is false, and  $\{\{a, f\}, \{b, f\}\}$ where $\K f$ is true and $\K e$ is false.
The presence of two answer sets in each world view is due to the cycle on objective atoms, whereas the presence of two world views is due to the cycle on subjective atoms (in general, the existence and number of world views are related to such cycles, see~\cite{Costantini19} for a detailed discussion).

\section{Epistemic Logic Programs: Useful Properties}
\label{elp-p}

As argued by Cabalar et al., it would be useful if ELPs would enjoy, \emph{mutatis mutandis},  properties similar to those of ASP programs.
Hence, in their works, such useful properties are outlined and adapted, as we report (almost literally) below. 


Drawing inspiration from the \emph{Splitting Theorem} \cite{Lif94},
an analogous property is defined for ELPs:
\begin{definition}[Epistemic splitting set {\protect\cite[Def.~4]{CabalarFC21}}]\label{epsplit}
A set of atoms $U \subseteq At$ is said to be an epistemic splitting set of a program $\Pi$ if for any rule $r$ in $\Pi$ one of the following conditions hold:
\begin{enumerate}
\item\label{enum:1}
$\mathit{Atoms}(r) \subseteq U$;
\item\label{enum:2}
$(\Bodyobj(r) \cup \mathit{Head}(r))\ \cap\ U = \emptyset$.
\end{enumerate}
An epistemic splitting of $\Pi$ is a pair $\langle B_U(\Pi),T_U(\Pi)\rangle$ such that $B_U(\Pi) \cap T_U(\Pi) = \emptyset$ and $B_U(\Pi) \cup T_U(\Pi) = \Pi$, and also, such that all rules in $B_U(\Pi)$ satisfy condition (\ref{enum:1}) and all rules in $T_U(\Pi)$ satisfy condition (\ref{enum:2}).
\end{definition}

Intuitively, condition~(\ref{enum:2}) means that the top program $T_U(\Pi)$  may refer to atoms in $U$ which occur as heads of rules in the bottom $B_U(\Pi)$,
only through epistemic operators.

Epistemic splitting can be used, similarly to ``traditional'' Lifschitz\&Turner splitting, for iterative computation of world views.
Indeed, Cabalar et al.~\citeyear{CabalarFC21} propose to compute first the world views of the bottom program $B_U(\Pi)$ and, for each of them, simplify the corresponding subjective literals in the top part. Given an epistemic splitting set $U$ for $\Pi$ and a set of interpretations $W$, they define the subjective reduct of the top with respect to $W$ and signature $U$, denoted by $E_U(\Pi,W)$. This operator 
considers all subjective literals $L$ occurring in $T_U(\Pi)$, such that the atoms occurring in them belong to $B_U(\Pi)$. In particular, $L$ will be substituted by $\top$ in $E_U(\Pi,W)$ if $W \models L$, and by $\bot$ otherwise. Thus, $E_U(\Pi,W)$ is a version of $T_U(\Pi)$ where some subjective literal, namely those referring to the bottom part of the program, have been simplified as illustrated.

\begin{definition}[{\protect\cite[Def.~5]{CabalarFC21}}]
Given a semantics $\cal{S}$, a pair $\langle W_b,W_t \rangle$ is said to be an $\cal{S}$-solution of $\Pi$ with respect to an epistemic splitting set $U$ if $W_b$ is an $\cal{S}$-world view of $B_U(\Pi)$ and $W_t$ is an $\cal{S}$-world view of $E_U(\Pi,W_b)$.	
\end{definition}

The definition is parametric w.r.t.~$\cal{S}$, as each different semantics $\cal{S}$ will define in its own way the $\cal{S}$-solutions for a given $U$ and $\Pi$. 

\begin{definition}\label{wbt}
The WBT operation $W_b \sqcup W_t$ on sets of propositional interpretations $W_b$ and~$W_t$ is defined as follows:
$$ W_b \sqcup W_t = \{ I_b \cup I_t | I_b \in W_b \wedge I_t \in W_t\}.$$
\end{definition}

\noindent We report from \cite{CabalarFC21} the definition of the following property:
\begin{property}[Epistemic Splitting Property (ESP)]
	\label{epsplitprop}
	A semantics $\cal{S}$ satisfies the epistemic splitting property if for any epistemic splitting set $U$ of any program $\Pi$: $W$ is an $\cal{S}$-world view of $\Pi$ iff there is an $\cal{S}$-solution $\langle W_b,W_t \rangle$ of $\Pi$ w.r.t.~$U$ such that $W = W_b \sqcup W_t$.
\end{property}

Then, under a semantics that satisfies ESP, world views of the entire program are obtainable as the union of world views of the bottom with world views of a simplified version of the top. The Epistemic Splitting Property implies \emph{Subjective Constraint Monotonicity}, i.e., for any epistemic program $\Pi$ and any subjective constraint $r$, $W$ is a world view of
$\Pi \cup \{r\}$ iff both $W$ is a world view of $\Pi$ and $W$ satisfies $r$.

As discussed in \cite{CabalarFC21}, 
many semantics do not satisfy the ESP property, which is in fact satisfied only by the very first semantics of ELPs, proposed in \cite{Gelfond91} and thus called G91 (and in some of its generalizations), and by Founded Autoepistemic Equilibrium Logic (FAEEL), defined in~\cite{CabalarFC19}.

Another interesting property is \emph{foundedness}.
Again, such a notion has been extended from objective programs (see  \cite[Def.~15]{CabalarFC21}).
Intuitively, a set $X$ of atoms is \emph{unfounded} w.r.t.\ an (objective) program $\Pi$ and an interpretation $I$,
if for every $A \in X$ there is no rule $r$ in $\Pi$ by which $A$ might be derived, without incurring in positive circularities and without forcing the derivation of more than one atom from the head of a disjunctive rule (see, e.g., \cite{LeoneRS97} for a formal definition).
For ELPs, one has to consider that unfoundedness can originate also from positive dependencies on positive subjective literals, like, e.g., in the program $A \ar \K A$.
Among the existing semantics, only FAEEL satisfies foundedness.

An interesting class of programs admitting a unique world view is characterized by the following definition.

\begin{definition}[Epistemic Stratification {\protect\cite[Def.~6]{CabalarFC21}}]\label{eps}
We say that an ELP $\Pi$ is epistemically stratified if we can assign an integer mapping $\lambda : At \rightarrow N$ to each atom (occurring in the program) such that:
\begin{itemize}
	\item 
	$\lambda(a) = \lambda(b)$ for any rule $r \in \Pi$ and atoms $a,b \in (\mathit{Atoms}(r) \setminus \Bodysubj(r))$, and
	\item
	$\lambda(a) > \lambda(b)$ for any pair of atoms $a,b$ for which there exists a rule $r \in \Pi$ with $a\in (\mathit{Head(r)} \cup \Bodyobj(r))$ and $b \in \Bodysubj(r)$.
\end{itemize}	
\end{definition}

\section{Observations}
\label{observations} 

The subdivision of an ELP into layers 
suggests that, in the upper layer, epistemic literals referring to the lower layer may be aimed at performing some kind of meta-reasoning about that layer.
If the epistemic splitting property is enforced, however,
 meta-level reasoning is in practice prevented. 
 This is so because if the semantics satisfies such property, then, it is the lower layer that determines
 the truth value of the subjective literals that connect the two layers.
In fact, according to Property~\ref{epsplitprop}, through the simplification w.r.t.\ the answer sets of the lower layer,
the upper layer is strongly (maybe sometimes too strongly) constrained.

For instance, let us consider the program $\Pi_0 = \{a\,|\,b,\ \bot \ar{\no}\K a\}$.
We can see that,
while the lower level $\{a\,|\,b\}$, considered as a program \emph{per se}, has the unique world view $\{\{a\},\{b\}\}$,
the overall program has no world views. In fact, $\K a$  does not hold in $\{\{a\},\{b\}\}$, thus the constraint is violated.

Notice, however, that the world view $\{\{a\}\}$ is instead accepted by some semantics, 
such as those defined in \cite{Gelfond11} and in \cite{ShenE16}, that do not satisfy the epistemic splitting property.
This world view may be seen as corresponding to an approach where the upper layer, in order to retain consistency, ``requires'' the lower layer to entail $a$, which is absolutely feasible by choosing $a$ over $b$ in the disjunction.

From this perspective,
the knowledge modeled by the upper layer is not just used to reject potential world views of the bottom level, 
but, instead, can affect the way in which they are composed, by filtering out some of the answer sets.
This situation is reminiscent of what actually happens for ASP: consider the  plain ASP program
\(\{a\,|\,b,\ c \ar a,\ \ar{\no} c\}\), which has unique answer set $\{a,c\}$, originating from the answer set $\{a\}$ of the 
lower layer \(\{a\,|\,b\}\).

We follow (for a long time) the line, amply represented in the literature, in which meta-reasoning is aimed 
not only at ``observing'' lower layer(s) but also at trying to influence them (cf.\ \cite{Costantini02} for a survey on meta-reasoning in Computational Logic); this by suitably enlarging and/or restricting, as an effect of meta-rules application, the set of possible consequences of such layer(s). We discuss at length this point of view, also proposing technical solutions and several examples, in \cite{CostantiniF21}.

In addition, let us notice that a common approach in logical declarative modeling of a problem consists of formalizing the problem domain as the ``top'' part of a program/theory. Then, such top part will be joined with a specific \as bottom'', representing the problem instance at hand,
that may vary and might be, in general, unknown while defining the top.

Below is an example of what we mean (over-simplified and in \as skeletal form'' for the sake of conciseness), taken from the realm of digital investigations, that the authors have been studying in the context of the Action COST CA17124  DIGital FORensics: 
evidence Analysis via intelligent Systems and Practices (DigForASP).
In the example, an investigation {\bf must} be concluded with a judgment, that can be:
\begin{itemize}
	\item
		of innocence if in no plausible scenario (i.e., in no answer set) evidence can be found of an involvement; 
	\item 
		of demonstrable guilt if in every possible scenario, the evidence of guilt can be found;
	\item
		of presumed innocence otherwise.
\end{itemize}
Clearly, the specification of the legal rules that can be used to draw conclusions, and then the details of each specific case will be modularly added whenever needed to this general \as top'' part. Thus, one can see a program composed of three layers: the top, and a bottom that can be further split into a middle layer containing legal rules, and the lowest layer with details of the case (see \cite{Costantini19} for more examples taken from this field). 

\noindent
The top layer is as follows:
$$
~~\begin{array}{l}
	\mathit{judgement} \ar \mathit{guilty}.\\
	\mathit{judgement} \ar \mathit{presumed\_innocent}.\\
	\mathit{judgement} \ar \mathit{innocent}.\\
	\ar not\ \K\ \mathit{judgement}.
\\ 
	\mathit{guilty} \ar \mathit{provably\_guilty}.\\
    \mathit{presumed\_innocent} \ar \no \mathit{provably\_guilty}.\\
	\mathit{provably\_guilty} \ar \K\ \mathit{sufficient\_evidence\_against}.\\
	\mathit{innocent} \ar \K\ \no \mathit{sufficient\_evidence\_against}.
\end{array}$$

\smallskip

Hence, a study of how the semantics of any resulting overall program might be built is in order here, as in many other practical cases: think, for example, of a top part comprising ontological definitions reusable in several application contexts. In fact, being able to compute and check a program's semantics only in dependence on each specific instance, does not seem to be elaboration-tolerant.

Therefore, we tried to understand whether the concept of splitting might be applied top-down, and how the existing semantics would behave in the new perspective.

\section{Our Proposal}\label{sect:proposal}

Let us proceed step by step towards the new definition of \emph{Top-down Epistemic Splitting Property}.
We first reformulate definitions related to ESP so that it can be applied also top-down, to obtain what we call Top-down Epistemic Splitting Property - Basic (TDESPB), showing that a semantics satisfies TDESPB if and only if it satisfies ESP.
Thus, TDESPB provides a way of coping with incremental computation of world views more suitable to the examples mentioned earlier.
We then perform some extensions, to obtain a more general Top-down Epistemic Splitting Property (TDESP) that holds for a wider range of semantic approaches.

\subsection{Preliminaries and Key Definitions}\label{sect:multisplit}

In our approach, the notion of splitting set remains the same, save for some details concerning subjective constraints.
We need, in fact, to introduce preliminary assumptions on constraints.
Notice that subjective literals may either occur in a subjective constraint directly or affect constraint's satisfaction through indirect dependencies, such as, e.g., in the program $\bot \ar a.\ a \ar \K p$
(see \cite{Dix95AeB} for a formal definition of direct and indirect dependencies).
Without loss of generality, we exclude here indirect dependencies concerning subjective literals involved in constraints.
Also, notice that, as it is well-known, a constraint can be represented as a unary odd cycle, that, e.g., for $\bot\ar \K p$ would be of the form $a \ar \no a, \K p$ (with $a$ introduced as a fresh atom), or even (as discussed in depth in \cite{Cos06}) as an odd cycle of any arity, of which $\K p$ is the unique \emph{handle}.
For the sake of simplicity, we consider subjective constraints in their plain form, namely, as in $\bot\ar \K p$.
Notice also that, according to the definition of splitting provided in \cite{CabalarFC21}, subjective constraints
can be placed at either of two adjacent levels. 
For convenience concerning definitions that will be introduced later, we impose, again without loss of generality, that 
both subjective rules satisfying condition~(\ref{enum:2}) of the definition of Epistemic Splitting Set (Definition\rif{epsplit}) and subjective constraints are put in~$T_U(\Pi)$.

We now proceed to introduce the key definitions on which our approach is based.

\begin{definition}
\label{fu}
Let be given a semantics $\cal{S}$, a program $\Pi$, and an epistemic splitting $\langle B_U(\Pi),T_U(\Pi)\rangle$ of~$\Pi$,
according to the definition of Epistemic Splitting Set.
Let $F_U(\Pi)$ denote the set of all subjective literals $\K L$ occurring in $T_U(\Pi)$ (even in negative form $\no \K L$) and referring to $B_U(\Pi)$ (in the sense that the atom involved in $\K L$ occurs in $B_U(\Pi)$ but not in $T_U(\Pi)$), together with their negations $\no \K L$. 
\end{definition}

Intuitively, subjective literals in $F_U(\Pi)$ constitute the \as interface'' between the top and bottom parts.
Notice that $\mathit{Atoms}(F_U(\Pi))\subseteq U$.

\begin{definition}
\label{tuprime}
Let $\Pi$ be a program and let
$F_U(\Pi)=\{\K L_1,\ldots,\K L_z, \no \K L_1,\ldots,\no \K L_z\}$. 
Let, moreover, $f_U(\Pi)=\{kl_1,\ldots,kl_z,nkl_1,\ldots,nkl_z\}$ be a set of fresh atoms.
The \emph{detached version}  $T'_U(\Pi)$  of $T_U(\Pi)$ is
the program consisting of:
\begin{itemize}
\item
the rules obtained from rules in $T_U(\Pi)$ by substituting each occurrence of the subjective literal $\K L_i\in F_U(\Pi)$ or $\no \K L_i\in F_U(\Pi)$
by the corresponding fresh atom $kl_i\in f_U(\Pi)$ or $nkl_i\in f_U(\Pi)$, for each $i\in\{1,\ldots,z\}$ (where $kl_i$ and $nkl_i$ are in turn called the \emph{detached form} of $\K L_i$ and $\no \K L_i$, resp.); ~ and
\item 
	the facts $kl_i\ |\ nkl_i$, for each $i\in\{1,\ldots,z\}$.
\end{itemize}
\end{definition}

We introduced $T'_U(\Pi)$ in order to model the connection between $T_U(\Pi)$ and $B_U(\Pi)$
w.r.t.\  the top-down perspective.
Thus, we need to define the notion of \emph{world views of the detached version $T'_U(\Pi)$ of a program} under the assumption that the fresh atoms $kl_i$ and $nkl_i$ represent the epistemic literals connecting the top and bottom parts of the program. 
As seen below, these world views not necessarily coincide with the world views
of  $T'_U(\Pi)$ if considered as an epistemic program by itself.

Recall that a disjunction between an epistemic literal $\K L$ and its negation $\no \K L$ determines, as discussed in \cite{Costantini19}, two world views, one entailing $\K L$ and the other one entailing $\no \K L$.
With respect to the subjective literals in $F_U(\Pi)$,
in defining the detached version $T'_U(\Pi)$ of a program  $T_U(\Pi)$  we encoded the potential existence of such alternative world views
by means of the disjunctions $kl_i\ |\ nkl_i$, for $i\in\{1,\ldots,z\}$.

In computing the world views of the detached version  $T'_U(\Pi)$,
we start by considering  $T'_U(\Pi)$ as a regular epistemic program (forgetting for the moment that the fresh atoms $kl_i$ and $nkl_i$ stand for epistemic literals) thus obtaining 
the corresponding collection of world views $\mathcal{W}$. Note in fact that $T'_U(\Pi)$
does not contain subjective literals referring to the bottom $B_U(\Pi)$, 
but it may contain \as local'' epistemic literals that may determine the existence of several world views (or just one if there are no such local epistemic literals).
The answer sets in each $W\in \mathcal{W}$ might however contain some of the atoms $kl_i$s and $nkl_i$s. 
In this case, each $W\in \mathcal{W}$ has to be split into two world views,
say $W_1$ and $W_2$, the former composed of the answer sets
in $W$ that contain $kl_1$, and the latter composed by those answer sets of $W$ that contain $nkl_1$.
This step must be repeated by considering the pair $kl_2$/$nkl_2$ in order to split both $W_1$ and $W_2$,
and so on, for each $i\in\{1,\ldots,z\}$.
(Observe that the order of splits does not matter.)
We consider the resulting collection of sets of atoms as the world views of the detached version  $T'_U(\Pi)$.
An example of this process will be given at the end of Section\rif{sect:tdespb}.
In summary:

\begin{definition}[World views of $T'_U(\Pi)$, or Interface World Views]\label{iww}
Let $W^1,\ldots,W^n$ be the world views of $T'_U(\Pi)$ according to a given semantics $\cal{S}$.
The Interface World Views of $T'_U(\Pi)$ 
are obtained as follows: for every $W^j$, $j\leq n$, $W^j = \{S^j_1,\ldots,S^j_v\}$ for some $v \geq 0$, and for every disjunction $kl_i\ |\ nkl_i$, $i\in\{1,\ldots,z\}$ occurring in $T'_U(\Pi)$, split $W^j$ into $W^j_1$ and $W^j_2$, the former composed of the sets $S^j_h \in W^j$  such that $kl_i \in S^j_h$, the latter composed of the of the sets $S^j_f \in W^j$  such that $nkl_i \in S^j_f$, $f \in \{1,\ldots,v\}$. Repeat the splitting over the resulting world views, and iterate the process until splitting is no longer possible, i.e., no resulting world view contains both $kl_r$ and $nkl_r$, for some $r\in\{1,\ldots,z\}$.
\end{definition}

The denomination ``Interface World Views'' indicates that they have been obtained in the perspective of a merge with world views of the bottom, as seen below. For the sake of conciseness though by some abuse of notation, we will call Interface World Views simply `world views'.

\begin{proposition}
There exists a bijection between world views of $T_U(\Pi)$ and world views of $T'_U(\Pi)$.
\end{proposition}
\begin{proof}
Given a world view (Interface World View, to be precise) $W'_j$ of the epistemic program $T'_U(\Pi)$, a world view $W_j$ for $T_U(\Pi)$ is equal to $W_j=\{X\setminus f_U(\Pi)\, |\, X\in W'_j \}$. In fact, the procedure for obtaining Interface World Views takes into account the fact that each epistemic literal represented by an atom in $f_U(\Pi)$ can be potentially either true or false.
Vice versa, $W'_j$ is obtained from $W_j$ by adding to it some subset of $f_U(\Pi)$. 
\end{proof}

For each of such world views $W_j$ of $T_U(\Pi)$, Def.~\ref{def:tdrs} below identifies the set of subjective literals that
are relevant in extending $W_j$ to a world view of the entire~$\Pi$.
These are those that in the detached version of $T_U(\Pi)$ have been assumed to be true to obtain $W_j$ as a world view.
\begin{definition}[Epistemic Top-down Requisite Set]\label{def:tdrs}
	Let $\langle B_U(\Pi),T_U(\Pi)\rangle$ be an epistemic splitting for a program $\Pi$,
	$W'_j$ be a world view of $T'_U(\Pi)$,
and  let $W_j=\{X\setminus f_U(\Pi)\, |\, X\in W'_j \}$.

The set $ES_{T_U(\Pi)}(W_j) = \{\K L_h \,|\, W'_j\models kl_h \} \cup \{\no \K L_h \,|\, W'_j\not\models kl_h \}$
		is the \emph{(epistemic top-down) requisite set} for~$W_j$ (w.r.t.\ $\langle B_U(\Pi),T_U(\Pi)\rangle$).
\end{definition}

Now we partition the \emph{requisite set}, identifying two relevant subsets (technical reasons for doing so 
will be seen below).
\begin{definition}\label{ecrq}
Given $f_U(\Pi)=\{kl_1,\ldots,kl_z,nkl_1,$ $\ldots,nkl_z\}$
and the above definition of requisite set $ES_{T_U(\Pi)}(W_j)$,
w.r.t.\ an  epistemic splitting $\langle B_U(\Pi),T_U(\Pi)\rangle$, 
let set $S$ include those $kl_i/nkl_i$ that occur in some constraints in $T'_U(\Pi)$.\\
We split the requisite set $ES_{T_U(\Pi)}(W_j)$ as the union of the following two (disjoint) sets:
\begin{itemize}
	\item
	the \emph{epistemic top-down constraint set}:
		$$EC_{T_U(\Pi)}(W_j)= (\{\K L_i\ | kl_i \in S\} \cup \{\no \K L_i\ | nkl_i \in S\})\cap ES_{T_U(\Pi)}(W_j)$$
	\item
	the \emph{requirement set}:
	$$RQ_{T_U(\Pi)}(W_j) = \big(\{\K L_i\ | kl_i \in f_U(\Pi){\setminus}S\} \cup
		\{\no\K L_i\ |n kl_i \in f_U(\Pi){\setminus}S\}\big)\cap ES_{T_U(\Pi)}(W_j).$$
\end{itemize}
\end{definition}

There is an important reason for distinguishing these two subsets.
Namely, the literals in
$EC_{T_U(\Pi)}(W_j)$, if not entailed in some world view of the bottom part of the program, lead to a constraint violation
and cause the non-existence of world views of $\Pi$ extending~$W_j$.
Thus, $EC_{T_U(\Pi)}(W_j)$ expresses prerequisites on which epistemic literals must be entailed in a world view
of $B_U(\Pi)$, so that such world view can be merged with $W_j$ in order to obtain a world view of~$\Pi$. 
Instead, literals in $RQ_{T_U(\Pi)}(W_j)$, can be usefully exploited, as seen below, to drive the selection of which world view of the bottom can be combined with a given world view of the top.

For all the three sets (requisite set, constraint set, and requirement set) one can possibly list only the epistemic literals of $F_U(\Pi)$ required to be true, all the others implicitly required to be false.

Given a world view $W$ of $T_U(\Pi)$ and considering literals belonging to $EC_{T_U(\Pi)}(W)$ which occur in the bodies of rules in $B_U(\Pi)$, we introduce a simplification that can be performed and will turn out to be useful later on.

\begin{definition}[Top-down Influence]
Given a world view $W$ of $T_U(\Pi)$, and its corresponding top-down constraint set $EC_{T_U(\Pi)}(W)$, the $W$-tailored version $B_U^{W}(\Pi)$ of $B_U(\Pi)$ is obtained by substituting in $B_U(\Pi)$ all literals $\K L \in EC_{T_U(\Pi)}(W)$ by~$L$.
\end{definition}

The intuition behind the above definition is that, if $\K A$ is in $EC_{T_U(\Pi)}(W)$, then $A$ must necessarily belong to every answer set of a world view of the bottom that can be possibly merged with $W$ in order to obtain a world view of the overall program $\Pi$.
Hence, it is indifferent that in the body of rules of $B_U(\Pi)$ it occurs $A$ rather than $\K A$, if $\K A\in EC_{T_U(\Pi)}(W)$.
Substituting $\K A$ with $A$ can, however, be useful, as discovered during the development of the G11 \cite{Gelfond11} and K15 semantics \cite{KahlWBG015}, to \as break'' unwanted positive cycles among subjective literals, that might lead to \emph{unfounded} world views (cf.\ \cite[Def.~15]{CabalarFC21}).

In our approach, the notion of top-down influence provides, as seen by examples in the next section, an alternative perspective on how a world view of the bottom is obtained, and, in a sense, a re-interpretation of the notion of foundedness (to be formally elaborated in future work).

In the top-down approach that we are going to propose, the world views of a given program $\Pi$ are obtained
as a combination of world views of the top and world views of the bottom, like in the bottom-up approach.
In the basic version of the Top-down Epistemic Splitting Property, presented in Section\rif{sect:tdespb}, there is only a change of perspective and a simple condition to drive the combination via the WBT operation (cf.\ Definition\rif{wbt}).

In the definition of the more general Top-down Epistemic Splitting Property, presented in Section\rif{tdesp}, one can notice two relevant changes: (i) the notion of top-down influence is exploited in the definition of candidate world views; (ii) a subset of a world view of the bottom (i.e., some of the answer sets occurring therein) may be cut out, so as to enable the merging via WBT with a ``compatible'' world view of the top.

Preliminarily:
\begin{definition}\label{fulfills}
	Given a set $E$ of epistemic literals 
	and a set of sets of atoms $W$, we say that $W$ \emph{fulfills}~$E$ iff
$\forall\, \K L \in E, W \models L$ and
$\forall\, \no \K L \in E, W \not\models L$.
\end{definition}

\subsection{Top-down Epistemic Splitting Property~-~Basic (TDESPB)}
\label{sect:tdespb}

\begin{definition}[Candidate World View - Basic Version]\label{cww}
	Given an epistemic splitting $\langle B_U(\Pi),T_U(\Pi)\rangle$ for a program~$\Pi$,
	let $W_T$ be a  world view of $T_U(\Pi)$ and let $W_B$ be a world view of $B_U(\Pi)$ that fulfills $EC_{T_U(\Pi)}(W_T)$ such that $W_B$ also fulfills $RQ_{T_U(\Pi)}(W_T)$ (overall,  $W_B$ fulfills the requisite set $ES_{T_U(\Pi)}(W_T)$).
	Then, 
	$$W = W_B \sqcup W_T = \{ I_b \cup I_t | I_b \in W_B \wedge I_t \in W_T \}$$
	is a \emph{candidate world view} for $\Pi$ (obtained from $W_T$ and $W_B$).
\end{definition}

It is possible that no world views of the bottom comply with the conditions posed by world views of the top:
in such case, $\Pi$ has no candidate world views.

We can now state a property that, if satisfied by a semantics, allows world views to be computed top-down:
\begin{definition}[Top-down Epistemic Splitting Property - Basic Version (TDESPB)]
	A semantics $\cal{S}$ satisfies \emph{basic top-down epistemic splitting}
	if any candidate world view of $\Pi$ according to Definition\rif{cww} is indeed a world view of $\Pi$ under $\cal{S}$.
\end{definition}

Below we show that TDESPB is equivalent to the Epistemic Splitting Property by \cite{CabalarFC21}, in the sense that both definitions are satisfied by the same semantic approaches and thus characterize the same world views.

\begin{theorem}[Equivalence ESP - TDESPB]\label{thm:tdespb}
	A semantics $\cal{S}$ satisfies TDESPB if and only if $\cal{S}$ satisfies the Epistemic Splitting Property ESP as defined in Definition\rif{epsplitprop}.
\end{theorem}
\begin{proof}

\emph{If part.} Assume that a given semantics $\cal{S}$ satisfies TDESPB.
To show that $\cal{S}$ satisfies ESP as well, it suffices to observe that the couple $\langle W_B,W_T \rangle$ according to Definition\rif{cww} is a $\cal{S}$-solution as required by the definition of ESP.
In fact, $W_B$ is an $\cal{S}$-world view of the bottom $B_U(\Pi)$.
It remains to be seen that $W_T$ is an $\cal{S}$-world view of $E_U(\Pi,W_B)$, i.e., that, after simplifying $T_U(\Pi)$ w.r.t.\ the subjective literals entailed by $W_B$, one would have $W_T$ among the world views. By Definition\rif{cww}, $W_B$ fulfills the requisite set $ES_{T_U(\Pi)}(W_T)$, leading $W_B \sqcup W_T$ to be a world view of the overall program. This means, according to Definition\rif{def:tdrs}, that $W_B$ entails all the subjective literals of the form $\K A$ and $\no \K A$, that, in the detached version of $T_U(\Pi)$ (Definition\rif{tuprime}) have been assumed to be true (in their detached form) in order to obtain $W_T$ as a world view (according to $\cal{S}$). Thus, if one would simplify $T_U(\Pi)$ into $E_U(\Pi,W_B)$ by considering exactly those subjective literals as true and all the others as false, one would trivially obtain $W_T$ as the world view of $E_U(\Pi,W_B)$.

\emph{Only if part.} Assume that a given semantics $\cal{S}$ satisfies ESP. This means that there exists a $\cal{S}$-solution $\langle W_B,W_T \rangle$ that, via WBT, gives rise to the world views of the program. To be an $\cal{S}$-solution, $W_B$ must be a world view of the bottom, and $W_T$ a world view of $E_U(\Pi,W_B)$, i.e., of the top simplified w.r.t.~$W_B$. To find the correspondence with TDESPB, we have to ascertain that $\langle W_B,W_T \rangle$ gives rise to candidate world views in the sense of Definition\rif{cww}. To do so, we put into $ES_{T_U(\Pi)}(W_T)$ the subjective literals, among those entailed by $W_B$, that are employed to perform such simplification, so as to exactly fulfill the conditions posed in Definition\rif{def:tdrs}.
\end{proof}

The equivalence stated by Theorem\rif{thm:tdespb} implies that the world views of a program can be determined by composing the world views of the various layers into which the program can be split, by proceeding either bottom-up, according to the original definition, or top-down, according to our new definition. We will now illustrate the approach, and its similarities and differences w.r.t.~ASP, by means of an example.

Consider the following sample ASP program.
	\[\begin{array}{ll}
	f \ar a.\\
	e \ar c.\\
	\bot \ar \no p.\\
	a \ar p.\\
	a \ar q.\\
	p \ar \no q.\\
	q \ar \no p.\\
	c.
\end{array}\]

\noindent
A possible split according to Lifschitz \& Turner can be:
\[\begin{array}{ll}
	\mathit{Top\ part}\\
	f \ar a.\\
	e \ar c.\\
	\bot \ar \no p.\\
	\\
	\mathit{Bottom\ part}\\
	a \ar p.\\
	a \ar q.\\
	p \ar \no q.\\
	q \ar \no p.\\
	c.
\end{array}\]

Notice that the unique answer set of this program is $S=\{c,p,a,e,f\}$. The answer sets of the bottom part are: $S1=\{c,p,a\}$, $S2=\{c,q,a\}$.
The answer set of the top part, assuming $p$ true (otherwise the constraint is violated), is $S3?=\{e?,f?\}$, the question mark meaning that any of the two atoms can be true, according to the selected answer set of the bottom. In this simple case, we have to choose $S2$, which makes $p$ true, and, by imagining adding atoms in $S2$ as new facts in the top part, we get both $e$ and $f$, thus obtaining the answer set $S$. Let us now consider the top part as a standalone program:
	\[\begin{array}{ll}
	f \ar a.\\
	e \ar c.\\
	\bot \ar \no p.
\end{array}\]

This program in itself is inconsistent, but knowing that it is intended as the top part of a wider program, we can set the requirements for any bottom part, in the form of what we can call \emph{Epistemic top-down Constraint set} $EC = \{p\}$, i.e., $p$ must be true in an answer set of the bottom, for the top to be consistent. If we enrich the top as follows:
	\[\begin{array}{ll}
	f \ar a.\\
	e \ar c.\\
	\bot \ar \no p.\\\\
	p\ |\ nop.\\
	a\ |\ noa.\\
	c\ |\ noc.
\end{array}\]

We can compute all possible answer sets for the top part, by simulating possible values for atoms coming from the (still unknown) bottom. Each such simulation, e.g., assuming $a$ true and $c$ false, gives rise to a \emph{Requisite Set} $RQ$. Then, given a specific bottom program that one intends to add to the top, each answer set $M$ of the bottom that fulfills $EC$ can be combined with all the answer sets of the top that are compatible, in the sense that $M$ entails all literals in the corresponding $RQ$.

Let us now consider an ELP with a very similar structure.
\[\begin{array}{ll}
	\mathit{Top\ part}\\
	f \ar \K a.\\
	e \ar \K c.\\
	\bot \ar \no \K p\\
	\\
	\mathit{Bottom\ part}\\
	a \ar p.\\
	a \ar q.\\
	p \ar \no \K q.\\
	q \ar \no \K p.\\
	c.
\end{array}\]

Let us first proceed bottom-up, as dictated by the ESP definition.
The world views of the bottom, according to any existing semantics, are:
$W1 = \{\{c,p,a\}\}$, $W2 = \{\{c,q,a\}\}$. 

Below is the top part simplified w.r.t.~$W1$, with a unique resulting world view $\{\{e,f\}\}$.
\[\begin{array}{ll}
	f.\\
	e.
\end{array}\]

The top part simplified w.r.t.~$W2$ is reported below, with no world views as the constraint is violated:
\[\begin{array}{ll}
	\mathit{Top\ part\ w.r.t.\ W2}\\
	f.\\
	e.\\
	\bot \ar \top	
\end{array}\]

Therefore, the unique world view of the overall program is, by the WBT operation which reduces here to a simple union, $W = \{\{c,p,a,e,f\}\}$.

Let us now apply the notions related to the top-down splitting property TDESPB that we presented above.
We have the following detached version of the top part:
	\[\begin{array}{ll}
	f \ar ka.\\
	e \ar kc.\\
	\bot \ar nkp\\
	ka\ |\ nka.\\
	kc\ |\ nkc.\\
	kp\ |\ nkp.\\
\end{array}\]

Seen as an epistemic program by itself, this program has a unique world view (indeed,
it is a standard ASP program),
which is
$$\{\{kp,nka,nkc\}, ~  \{kp,ka,nkc,f\}, ~  \{kp,nka,kc,e\},  ~ \{kp,ka,kc,e,f\}\}.$$
\noindent
By splitting this set of sets three times (w.r.t.\ the pairs $ka/nka$, $kc/nkc$, and $kp/nkp$)
as described in Section\rif{sect:multisplit}, Definition\rif{iww},
we obtain the world views of the detached version: 
$\{\{kp,nka,nkc\}\}$,
$\{\{kp,nka,kc,e\}\}$,
$\{\{kp,ka,nkc,f\}\}$, and
$\{\{kp,ka,kc,e,f\}\}$.

From them, one determines the \emph{epistemic top-down constraint set} which is, clearly, $EC = \{\K p\}$, stating that the unique constraint must be satisfied. Any {\bf compatible} world view of a bottom should satisfy one of the $RQ^i$'s, $i\leq 4$, i.e., the requirement sets, listed below (cf. Definitions\rif{def:tdrs} and\rif{ecrq}). To each $RQ^i$ it corresponds a world view of the top (indicated on the right) to be united to those world views of the bottom that satisfy $RQ^i$ (if any).
\[\begin{array}{ll}	
	RQ^1 = \emptyset, & \mathit{determines}\tbm W_t^1 = \{\emptyset\}\\
	RQ^2 = \{\K c\}, & \mathit{determines}\tbm W_t^2 = \{\{e\}\}\\
	RQ^3 = \{\K a\}, & \mathit{determines}\tbm W_t^3 = \{\{f\}\}\\
	RQ^4 = \{\K c, \K a\}, & \mathit{determines}\tbm W_t^4 = \{\{e,f\}\}
\end{array}\]

Given the world views of the bottom, i.e.,
$W1 = \{\{c,p,a\}\}$, $W2 = \{\{c,q,a\}\}$, we can see that $W2$ does not fulfill $EC$ and so must be discarded, while $W1$ fulfills $EC$ and also $RQ^4$, thus leading, by the WBT operation which reduces here to a simple union, the to (unique) world view of the overall program $W = \{\{c,p,a,e,f\}\}$.

\medskip
It is immediate to verify that the result obtained via the bottom-up and the top-down approach is indeed the same.

\subsection{Top-down Epistemic Splitting Property (TDESP)}
\label{tdesp}

In this subsection, we will extend previous definitions to a more general form, so as to be able to characterize in a top-down fashion the world views obtained according to many semantic approaches presented in the literature, other than G91 and FAAEL, such as, e.g., those proposed by \cite{ShenE16,Kahl18,Su19}; in fact, they do not enjoy the basic property TDESPB illustrated above. We introduce a different way of computing candidate world views, where, in the absence of a world view of the bottom that fulfills the set $EC$ relative to the top, one can select a subset of such a world view. This, as we will demonstrate in our running example, is analogous to what is customarily done in~ASP.

\begin{definition}[Candidate World View]\label{cwwext}
Given an epistemic splitting $\langle B_U(\Pi),T_U(\Pi)\rangle$ for a program~$\Pi$,
let $W_T$ be a  world view of $T_U(\Pi)$ and let $W_B$ be a subset of a world view of $B_U^{W_T}(\Pi)$ that fulfills $EC_{T_U(\Pi)}(W_T)$ (where, if $EC$ is empty, $W_B$ is the entire world view of the bottom) such that $W_B$ fulfills $RQ_{T_U(\Pi)}(W_T)$.
	Then, 
	$$W = W_B \sqcup W_T = \{ I_b \cup I_t | I_b \in W_B \wedge I_t \in W_T \}$$
	is  a \emph{candidate world view} for $\Pi$ (obtained from $W_T$ and $W_B$).
\end{definition}

Note that, candidate world views are now computed after applying top-down influence.
It is possible that no subset of any world view of the bottom complies with the conditions posed by world views of the top.
In such case, $\Pi$ has no candidate world views.

We can now state another property concerning top-down epistemic splitting that a semantics might obey:
\begin{definition}[Top-down Epistemic Splitting Property (TDESP)]
	A semantics $\cal{S}$ satisfies \emph{top-down epistemic splitting}
	if any candidate world view of $\Pi$ according to Definition\rif{cwwext} is indeed a world view of $\Pi$ under $\cal{S}$.
\end{definition}

We can state the relationship among TDESP and ESP/TDESPB (that, as seen, are equivalent).

\begin{theorem}
Given a semantics $\cal{S}$ satisfies both foundedness and ESP/TDESPB, then $\cal{S}$ satisfies TBDESP.
\end{theorem}

\begin{proof}
If $\cal{S}$ satisfies TDESPB, this means that for every world view of given program $\Pi$ obtained via the WBT operation, and thus composed of a world view $W_T$ of the top and a world view $W_B$ of the bottom, every $\K L \in EC_{T_U(\Pi)}(W)$ is entailed by $W_B$ and, if $\cal{S}$ satisfies foundedness, this equates to say that $L$ is entailed by the bottom part of the program. Thus, the application of Top-down Influence is irrelevant. We then notice that, according to Definitions\rif{cwwext} and\rif{tdesp} a candidate world view for TDESP can be obtained from an entire world view of the bottom, as done for TDESPB. This concludes the proof, showing that for this class of semantics TDESP and TDESPB are indeed equivalent. 
\end{proof}

The above theorem is immediately applicable to the FAAEL semantics. For semantics which do not enjoy foundedness things are different, as seen in the examples below.
We will now, in fact, experiment with our methodology on some relevant examples proposed in recent literature.
Consider program $\Pi_1$, taken from \cite{ShenEiter2020}:
$$\begin{array}{llcll}p\ |\ q	&~ (r1)\\
\bot \ar \no \K p &~(C)
\end{array}$$

Here, $B_U(\Pi_1)$ consists of rule (r1), and $T_U(\Pi_1)$ consists of constraint (C). So, $T'_U(\Pi_1)$ is (where $kp$ and $nkp$ are fresh atoms):
$$
\begin{array}{llcll}
	kp\ |\ nkp	&~ (r1)\\
	\bot \ar nkp & 
\end{array}$$

\noindent
whose unique world view is $\{\{kp\}\}$. After canceling $kp$, we obtain $W_T = \{\emptyset\}$ for $T_U(\Pi_1)$, with $ES_{T_U(\Pi_1)}(W_T) = EC_{T_U(\Pi_1)}(W_T) = \{\K p\}$ and $RQ_{T_U(\Pi_1)}(W_T) = \emptyset$. Regardless of the epistemic semantics $\cal{S}$, as no subjective literals occur therein, the unique world view of $B_U(\Pi_1)$ is $\hat{W}=\{\{p\}, \{q\}\}$. Since $W_B = \{\{p\}\}$ is the only subset of $\hat{W}$ fulfilling $EC_{T_U(\Pi_1)}(W_T)$ (cf.\ Definition\rif{cwwext}), then it is the one selected.
It is also a world view for $\Pi_1$, as the unique world view of the top part is empty.
This world view violates subjective constraint monotonicity, still, it is the one delivered by the semantics proposed in \cite{ShenE16} and, as noticed in \cite{ShenEiter2020}, by those proposed in \cite{Kahl18,Su19}.

In our opinion the world view $\{\{p\}\}$ captures the ``intended meaning'' of the program $\Pi_1$,
where the top layer \as asks'' the bottom layer to support, if possible, $\K p$ (in order not to make the overall program inconsistent). Let us, in fact, introduce a simple variation, by adding a fact, say $c$, to the program, where $c$ also occurs in the constraint, obtaining:
$$\begin{array}{llcll}		p\ |\ q	&~ (r1)\\
		c. &~ (f1)\\
		\bot \ar c, \no \K p &~(C)
	\end{array}$$

We would obtain, in this case, the world view $\{\{c,p\}\}$. Let us now reinterpret this program within the work of the COST Action DigForASP, i.e., in the realm of investigations. A rephrasing could be the following:
$$\begin{array}{llcll}		at\_crime\_scene\ |\ not\_at\_crime\_scene	&~ (r1)\\
		reliable\_witness\_recognizes. &~ (f1)\\
		\bot \ar reliable\_witness\_recognizes, \no \K\, at\_crime\_scene &~(C)
	\end{array}$$

The meaning underlying the schematic formulation is that it is uncertain whether a suspect was or not at the crime scene. However, if a reliable witness recognized the suspect, then investigators can be certain that the suspect was indeed at the crime scene. The constraint could in fact be rephrased (although this is not legal syntax) into:
$$\begin{array}{llcll}		\K\, at\_crime\_scene \ar reliable\_witness\_recognizes.
	\end{array}$$

The use of the $\K$ is crucial here because one wants to distinguish between facts collected by the investigators and reliable conclusions derived by these facts. Thus, the world view $\{\{reliable\_witness\_recognizes,at\_crime\_scene\}\}$ makes perfect sense here.

In addition, one might consider the very similar ASP program:
$$\begin{array}{llcll}		p\ |\ q	&~ (r1)\\
		c. &~ (f1)\\
		\bot \ar c, \no p &~(C)
	\end{array}$$

\noindent
with unique answer set $\{c,p\}$. The \as bottom'' program fragment consisting of (r1)+(f1) would also have answer set $\{c,q\}$, which is discarded since it would lead to violating the constraint. We may consider this program as an ELP, with unique world view $\{\{c,p\}\}$ obtained from a subset of the world view $\{\{c,p\}, \{c,q\}\}$ of the bottom (union the empty world view of the top), exactly as specified in Definition\rif{cwwext}.

\medskip
Consider now the following program $\Pi_2$.
$$\begin{array}{ll}
	p\ |\ q	&~~~(r1)\\
	\bot \ar \no \K p  & ~~~(C)\\
p \ar \K q	&~~~(r2)\\
q \ar \K p	&~~~(r3)
\end{array}$$

\noindent
Here, $B_U(\Pi_2)$ consists of rules (r1-r3), and $T_U(\Pi_2)$ consists of constraint (C). So, $T'_U(\Pi_2)$ is
(where $kp$ and $nkp$ are fresh atoms):
$$\begin{array}{ll}	kp\ |\ nkp\\
	\bot \ar nkp 
\end{array}$$

\noindent
whose unique world view is $\{\{kp\}\}$. After canceling $kp$, we obtain world view $W_T = \{\emptyset\}$ for $T_U(\Pi_2)$ where $ES_{T_U(\Pi_2)}(W_T) = EC_{T_U(\Pi_2)}(W_T) = \{\K p\}$ and set $RQ$ is empty. Regardless of  the semantics $\cal{S}$, the potential world views of $B_U(\Pi_2)$ are $W_1=\{\{p\}\}$, $W_2 = \{\{q\}\}$, $W_3=\{\{p\}, \{q\}\}$, $W_4=\{\{p,q\}\}$.
Actually, $W_4$ is the only one fulfilling $ES_{T_U(\Pi_2)}(W_T)$; $W_1$ has the problem that, having $p$ and fulfilling $\K p$, (r3) might be applied thus getting $q$.
Note that $W_4$ is in fact the world view returned by semantics proposed, for instance, in \cite{KahlWBG015} and \cite{ShenE16}.
It is easy to see that $W_4$ violates foundedness.
However, in our approach $q$ is not derived via the positive cycle (extended to subjective literals), but from the $\K p$ \as forced'' by the upper layer via top-down influence, which substitutes $\K p$ with $p$ in rule (r3) of $B_U(\Pi_2)$. This in a sense guarantees a form of foundedness, though not the formal one introduced in \cite[Def.~15]{CabalarFC21}. Since the unique world view for the top is empty, then the unique world view of the overall program is, indeed, according to our method, $W=W_4=\{\{p,q\}\}$. 

Let us now consider 
$\Pi_3$ to be the seminal example introduced in \cite{Gelfond91}, which is discussed in virtually every paper on ELP. 
$\Pi_3$ is epistemically stratified (see Definition\rif{eps} and~\cite[Def.~6]{CabalarFC21}).
This formulation (variations have appeared over time) is from~\cite{CabalarFC21}.
$$\begin{array}{ll}
\mathit{eligible(X)} \ar	\mathit{high(X)}	&~(r1)\phantom{\overline{\overline{|}}}\\
\mathit{eligible(X)}	\ar	\mathit{minority(X)}, \mathit{fair(X)}	&~(r2) \\
\mathit{noeligible(X)} \ar \no \mathit{fair(X)}, \no \mathit{high(X)}	&~(r3)\\
	\mathit{fair(mike)}\ |\ \mathit{high(mike)} & ~(f1)\\
\mathit{interview(X)} \ar  \no \K\, \mathit{eligible(X)}, \no \K\, \mathit{noeligible(X)} & ~(r4)\\
\mathit{appointment(X)} \ar \K\, \mathit{interview(X)} & ~(r5) \phantom{\underline{\underline{|}}}
\end{array}$$

\noindent 
Since in this version of the program we have only $\mathit{mike}$ as an individual, we may obtain the following ground abbreviated version:	
$$
	\begin{array}{lc}
		e \ar h &~(r1)\\
		e \ar m, f &~(r2)
		\\
		ne \ar \no f, \no h &~(r3)\\
		f\ |\ h & ~(f1)
		\\
		in \ar \no \K e,\; \no \K ne & ~(r4)\\
		a \ar \K in & ~(r5)  
\end{array}$$

\noindent
Here, we consider (r5) as the top $T_U(\Pi_3)$, and (r1-r4) plus (f1) as the bottom, which can be however in turn divided into the top $T1_U(\Pi_3)$ including (r4), and the bottom $B_U(\Pi_3)$, made of (r1-r3) and (f1). 
So, $T'_U(\Pi_3)$ is (with fresh atoms $kin$, $nkin$):
$$\begin{array}{llcll}	a \ar kin & ~(r5')\\
	kin\ |\ nkin\\
\end{array}$$

\noindent
with two answer sets: $\{a,kin\}, \{nkin\}$. As explained in Section\rif{sect:multisplit}, 
$kin\ |\ nkin$ stands for a disjunction between the epistemic literal $\K in$ and its negation $\no \K in$. This determines the existence of two world views, each entailing only one of these atoms, i.e. epistemic literals, where atom $a$ can, however, be derived only from the former. Thus, we have $W_{11} = \{\{a\}\}$ with $ES_{T_U(\Pi_3)}(W_{11}) = \{\K in\}$,
and $W_{12} = \{\emptyset\}$ with $ES_{T_U(\Pi_3)}(W_{12}) = \{\no \K in\}$.
$EC_{T1_U(\Pi_3)}$ is empty for all world views, as no constraint is present in $\Pi_3$.
Then, $T1'_U(\Pi_3)$ is (with $ke, nke, kne, nkne$ fresh atoms):
$$\begin{array}{llcllcll}
	in \ar nke, nkne & ~(r4')\\
	ke\ |\ nke\\
	kne\ |\ nkne.
\end{array}$$

\noindent
By the same reasoning as above, since there are two disjunctions among fresh atoms representing epistemic literals, four world views can be found. After canceling the fresh atoms, in fact we have $W_{21} = \{\{in\}\}$, with $ES_{T1'_U(\Pi_3)}(W_{21}) = \{\no \K e, \no \K ne\}$, and three empty world views $W_{22} = W_{23} = W_{24} =\{\emptyset\}$, with 
requisite sets
$\{\K e, \K ne\}$,
$\{\K ne, \no \K e\}$, and
$\{\no \K ne, \K e\}$, respectively. 
Clearly, also $EC_{T1'_U(\Pi_3)}$ is empty.

Finally, $B_U(\Pi_3)$, which is made of the rules (r1-r3) and (f1),
has the world view $W_3 = \{\{h,e\},\{f\}\}$. Since the requirement set relative to world view $W_{21}$ for the immediately upper level is satisfied in both answer sets of $W_3$, we can obtain an intermediate world view  $W_{{213}} = \{\{h,e,in\},\{f,in\}\}$ for the part of the program including (r1-r4). Considering also the top, it is easily seen that $W_{{213}}$ is compliant with the requirement set of $W_{11} = \{a\}$. So, we can obtain for the overall program the unique candidate world view $W= \{\{h,e,in,a\},\{f,in,a\}\}$, which is indeed a world view. Notice that, in fact, the world views that are part of the union, corresponding to the various sub-programs, would be the same under all known semantics for ELPs.  

Assume now that, instead of 
$f\ |\ h$, the program contains the bare fact $h$. Then, the world view of the bottom becomes $W_3 = \{\{h,e\}\}$. This world view implies $\K e$, so it can be combined with a world view $\{\emptyset\}$ of the middle layer, and since it also implies $\no \K in$, the further combination is with world view $W_{12} = \{\emptyset\}$ of the top. So, $W_3 = \{\{h,e\}$ is in this case the unique world view of the overall program.

\section{Main Result}
\label{theorem}

It is at this point interesting to try to assess formally which semantics (if any) satisfy the top-down epistemic splitting property TDESP.

We examine now the case
of the semantics introduced in \cite{KahlWBG015}, that we call for short K15. The reason for choosing K15 is that in \cite{CabalarFC21} it is noticed that K15 slightly generalizes the semantics proposed in \cite{Gelfond11} (called G11 for short) and can be seen as a basis for the semantics proposed in \cite{ShenE16} (called S16 for short). In particular, S16 (which considers instead of \K\ the operator \N\hspace{-0.1cm}$A$ which means $\no \K A$) treats K15 world views as candidate solutions, to be pruned in a second step, where some unwanted world views are removed by maximizing what is not known.
Thus, should K15 satisfy the top-down epistemic splitting property, S16 would do as well, and so would G11, the latter however only for the (wide) class of programs where its world views coincide with those of K15.
 
\begin{definition}[K15-world views]
	The 
 K15-reduct of $\Pi$ with respect to a non-empty set of interpretations $W$ is obtained by:
 \begin{itemize}
     \item[(i)] 
     replacing by $\bot$ every subjective literal $L \in \Bodysubj(r)$ such that $W \not\models L$, and
    \item[(ii)] 
    replacing all other occurrences of subjective literals of the form $\K L$ by $L$.
\end{itemize}
A non-empty set of interpretations $W$ is a K15-world view of $\Pi$ iff $W$ is the set of all stable models of the K15-reduct of $\Pi$ with respect to $W$.
\end{definition}	

We are able to prove the following:
\begin{theorem}[K15 TDESP]
	The K15 semantics satisfies the Top-down Epistemic Splitting Property. I.e., given an ELP $\Pi$, and set of sets $W$, where each set is composed of atoms occurring in $\Pi$, $W$ is a K15 world view for $\Pi$ if and only if it is a candidate world view for $\Pi$ according to Definition\rif{cwwext}.
\end{theorem}
\begin{proof}
Assume an Epistemic Splitting of given program $\Pi$ into two layers, top $T_U(\Pi)$ and bottom $B_U(\Pi)$ (where the reasoning below can, however, be iterated over a subdivision into an arbitrary number of levels). Notice that, given a K15 world view $W$, since each atom $A$ that occurs in the sets composing $W$ is derived in the part of the program including rules with head $A$, then $W$ can be divided into two parts, $W_T$, and $W_B$
	which are world views of $T_U(\Pi)$ and $B_U(\Pi)$, resp., each one composed of stable models of the K15-reduct of that part of the program.

\emph{If part.} Given a K15 world view $W$, let $Sl^T$ be the subjective literals occurring in $T_U(\Pi)$ which are entailed by the bottom, i.e., either of the form $\K A$, for which $W_B \models A$, or of the form $\no \K A$, for which $W_B \not\models A$.
Let such a set of literals form the set $ES_{T_U(\Pi)}(W_T)$.
(As mentioned, the subset of $Sl^T$ that consists of literals involved in constraints in $T_U(\Pi)$
	will form set $EC_{T_U(\Pi)}(W_T)$, and the remaining ones will form set $RQ_{T_U(\Pi)}(W_T)$.)
Therefore, we can conclude that $W$, which is a K15 world view, is indeed a candidate world view according to Definition\rif{cwwext}.

	\emph{Only if part.} Consider a candidate world view $W$ w.r.t.\ the K15 semantics, obtained by combining a subset $W_B$ of a K15 world view of $B_U(\Pi)$ with a K15 world view $W_T$ of $T_U(\Pi)$ after top-down influence. According to Definition\rif{cwwext}, the combination is possible only if for each epistemic literal $\K A \in ES_{T_U(\Pi)}(W_T)$, $W_B \models A$, and for each epistemic literal  $\no\K A \in ES_{T_U(\Pi)}(W_T)$, $W_B \not\models A$. If any such literal belongs to $EC_{T_U(\Pi)}(W_T)$, if this is not the case then there would be a constraint violation in $T_U(\Pi)$, so there would be no world views for $T_U(\Pi)$, and for the overall program $\Pi$. Considering a subjective literal in $RQ_{T_U(\Pi)}(W_T)$, if it were not the case that $W_B$ entails such literal, then by definition of K15 it would have been substituted by $\bot$, so $W_T$ would have been a different set. The top-down influence step can be disregarded since it performs in advance on elements of $ES_{T_U(\Pi)}(W_T)$, that are required to be entailed by $W_B$ anyway, the same transformation performed by K15, step~(ii).
Then, a candidate world view $W$ obtained according to Definition\rif{cwwext} is indeed a K15 world view.
\end{proof}

In \cite[Th.~2]{CabalarFC21} it is proved that, for any
semantics obeying epistemic splitting, an epistemically stratified program has a unique world
view. Actually, it can be seen that epistemically stratified programs admit one (and the same)
world view under any existing semantics, and in particular under those considered here: as it is well-known (see, e.g. \cite{Gelfond94,ShenE16,Costantini19}), multiple world views arise
in consequence of negative cycles involving epistemic literals, clearly not present in such programs. 
So, the unique world view of an epistemically stratified program is, in particular, a K15 world view. Thus, we have the following.

\begin{corollary}
Epistemically Stratified Programs satisfy both the Top-down and Bottom-up Epistemic Splitting Properties under any semantics.	
\end{corollary}

\section{Conclusions}
\label{conclusions}

In this paper, we have provided a way of exploiting the splitting of Epistemic Logic Programs in a top-down fashion, adequate for those situations where the top part of a program is well-established as it represents a problem formulation, where the bottom part (representing a problem instance) may vary and is in general not known in advance. 

We defined formal conditions for the combination of world views of the top with world views of the bottom into world views of the overall program. In addition, potential world views of the top can be pre-computed, thus simplifying the combination with the world views of each problem instance. We provide a version that is the top-down declination of the well-established approach by Cabalar et al., and a more general version that is applicable to a wider range of semantic approaches.

A question that may arise concerns the efficiency of the top-down approach, even though in many cases it will be an almost inevitable choice. If the subjective literals ``connecting'' adjacent layers are in small numbers (as it seems reasonable), then efficiency might not be a concern. 

It remains to be seen in more depth for which kinds of applications
the different approaches to splitting (top-down and bottom-up) might be most profitably exploited. As an example, we can go back to the suggestion proposed in \cite{KahlWBG015} to encode the problem of finding a conformant plan as the task of obtaining a world view. As emphasized in \cite{CabalarFC21}, splitting allows one to separate the planner definition (the \as top'') from the generation of alternative plans (an intermediate layer, we might say \as the top of the bottom'') from, in turn, the domain description (the \as bottom''). The top-down perspective would allow one to analyze the top part independently from the other layers, so as to identify in advance the prerequisites it poses to them.

An investigation of which other semantics might satisfy the Top-down Epistemic Splitting Property is also a subject of future work.


\end{document}